\documentclass[runningheads]{llncs}
\usepackage[T1]{fontenc}
\usepackage{graphicx}
\usepackage{booktabs}
\usepackage[misc]{ifsym}

\usepackage{mwe}

\usepackage[numbers]{natbib}
\usepackage{dirtytalk}
\usepackage{xspace}
\usepackage{soul}
\usepackage{thm-restate}
\usepackage{stmaryrd}
\usepackage{amsmath}
\usepackage{amssymb}
\usepackage{mathtools}
\usepackage{xcolor}

\usepackage{subfig}

\newcommand{\cf}{cf.\xspace}
\newcommand{\eg}{e.g.\xspace}
\newcommand{\ie}{i.e.\xspace}

\newcommand{\loss}{\ensuremath{L}}

\newcommand{\totalassignment}{\ensuremath{\mathbf{w}}}

\newcommand{\dataset}{\ensuremath{\mathcal{D}}}

\newcommand{\prob}{\ensuremath{p}}

\newcommand{\bvar}{\ensuremath{b}}
\newcommand{\bvars}{\ensuremath{\mathbf{b}}}

\newcommand{\xvars}{\ensuremath{\mathbf{x}}}
\newcommand{\yvar}{\ensuremath{y}}
\newcommand{\yvars}{\ensuremath{\mathbf{y}}}

\newcommand{\params}{\ensuremath{\theta}}

\newcommand{\tighteq}{\ensuremath{{=}}}
\newcommand{\ive}[1]{\llbracket#1\rrbracket}
\DeclareMathOperator*{\argmin}{arg\,min}
\DeclareMathOperator*{\argmax}{arg\,max}

\newcommand{\dsic}{DSIC\xspace}
\newcommand{\dscic}{DSCIC\xspace}
\newcommand{\nesy}{NeSy\xspace}

\begin{document}

\title{
Independence Is Not an Issue\\in Neurosymbolic AI}

\titlerunning{Independence Is Not an Issue in Neurosymbolic AI}

\author{Håkan Karlsson Faronius, Pedro Zuidberg Dos Martires}


\authorrunning{H. Karlsson Faronius and P. Zuidberg Dos Martires}


 \institute{AASS, Örebro University, Sweden}

\maketitle              

\begin{abstract}
A popular approach to neurosymbolic AI is to take the output of the last layer of a neural network, \eg a softmax activation, and pass it through a sparse computation graph encoding certain logical constraints one wishes to enforce.
This induces a probability distribution over a set of random variables, which happen to be conditionally independent of each other in many commonly used neurosymbolic AI models.
Such conditionally independent random variables have been deemed harmful as their presence has been observed to co-occur with a phenomenon dubbed \textit{deterministic bias}, where systems learn to deterministically prefer one of the valid solutions from the solution space over the others.
We provide evidence contesting this conclusion and show that the phenomenon of \textit{deterministic bias} is an artifact of improperly applying neurosymbolic AI.


\end{abstract}

\section{Introduction}
\label{sec:intro}

Neurosymbolic (\nesy) AI is an approach to AI which seeks to combine logic and neural networks \citep{garcez2023neurosymbolic}. 
Such an integration of symbolic and sub-symbolic methods allows, inter alia, for more interpretable \cite{koh2020concept} and data efficient \cite{diligenti2017semantic,manhaeve2018deepproblog} AI systems.

One of the most popular approaches to realize \nesy systems, uses the idea of a semantic loss function \cite{xu2018semantic, manhaeve2018deepproblog},
which imposes logical constraints on the outputs of a neural network while retaining  end-to-end differentiability.

\begin{description}
    \item[C1] As a first contribution we show that neurosymbolic AI formulated using the semantic loss can be seen as a special case of so-called \textit{disjunctive supervision} \cite{zombori2024towards}, \cf Section~\ref{sec:theory}.
\end{description}
Disjunctive supervision \cite{zombori2024towards} is a  setting for multi-class classification in which examples may be labeled with a disjunction of classes, \ie a single input can have multiple valid outputs.
Our result formally relates neurosymbolic AI to a wider range of techniques that are already being explored in the machine learning community. Note, however, that important differences exist and the often implicit assumptions can result in drastically differing outcomes.

\begin{description}
    \item[C2] In Section~\ref{sec:experiments} we experimentally show how the different assumptions in NeSy AI and disjunctive supervision learning drastically affect the behavior of classifiers when trained with weak supervision. 
\end{description}

In a recent study, van Krieken et al. \cite{van2024independence} reported that under assumptions commonly made in neurosymbolic AI, \nesy systems exhibit a phenomenon they dub \textit{deterministic bias}. More specifically, they write that {\em \say{the conditional independence assumption causes neurosymbolic methods to be biased towards deterministic solutions. This is because minima of neurosymbolic losses have to deterministically assign values to some variables.}}
We do not observe any such phenomenon in our experimental evaluation.  
On the other hand, we  report on a  phenomenon related to deterministic bias for disjunctive supervision, as predicted by Zombori et al. \cite{zombori2024towards}.

 \begin{description}
     \item[C3] In Section~\ref{sec:mortem} we identify that van Krieken et al. \cite{van2024independence} do not use the semantic loss as originally derived in \cite{xu2018semantic}, which takes into account both positive examples that satisfy the constraints and negative ones that do not. Instead they use a truncated semantic loss that takes into account positives only.
    This explains why we did not observe \textit{deterministic bias} in our experimental evaluation (\cf Section~\ref{sec:experiments}).
 \end{description}

\section{Preliminaries}
\label{sec:prelim}

\subsection{From Probabilistic Logic to Neurosymbolic AI}

As pointed out by Poole and Wood \cite{poole2022probabilistic} (and already earlier by Laplace (1814) \cite{laplace1814essai}, Poole (1993) \cite{poole1993probabilistic}, Pearl (2000) \cite{pearl2000models}, and Poole (2010) \cite{poole2010probabilistic}) probabilistic models consist of deterministic systems and independent probabilistic choices (\dsic). This simple principle can, for instance, be used to construct Turing complete probabilistic logic programming languages as done by Sato with his celebrated distribution semantics \cite{taisuke1995statistical}. The \dsic principle can even be used to construct probabilistic programming languages with an uncountable number of random variables \cite{dos2024declarative}.

In the context of logic programming, which many \nesy systems are based on \cite{manhaeve2018deepproblog,badreddine2022logic}, the \dsic principle implies that a probabilistic model consists of a set of (deterministic) logic formulas and a set of literals that are not deterministically true or false but are only true or false with a certain probability. This means in turn that each formula $\phi$ is only satisfied with a certain probability:
\begin{align}
    \prob(\phi \tighteq \top) = \sum_{\totalassignment} \prob(\phi\tighteq \top, \bvars\tighteq \totalassignment),
\end{align}
where $\totalassignment = (w_1, \dots, w_N)$ ($w_i \in \{\bot, \top\}$) is called a world and denotes a value assignment (either true or false) to all the $N$ Boolean variables $\bvars$ in $\phi$. The sum goes over all $2^N$ possible value assignments.
Using Poole's \dsic principle \cite{poole2010probabilistic} we further write:

\begin{align}
    \prob(\phi \tighteq \top)
    &=
    \sum_{\totalassignment}
    \prob(\phi \tighteq \top \mid \bvar\tighteq\totalassignment ) 
    \prob(\bvars\tighteq \totalassignment)
    \\
    &=
    \sum_{\totalassignment}
    \prob(\phi \tighteq \top \mid \bvars \tighteq\totalassignment )
    \prod_{i: w_i \in \totalassignment}
    \prob(\bvar_i\tighteq w_i)
\end{align}
where $w_i\in \{\top, \bot \}$ and where we have $\prob(\bvar_i \tighteq \top)= 1 {-} \prob(\bvar_i \tighteq \bot)$. For the sake of notational ease we will often write this as:
\begin{align}
    \prob(\phi)
    =
    \sum_{\totalassignment} \prob(\phi \mid \totalassignment )
\prod_{w_i \in \totalassignment} \prob( w_i).
\end{align}
Given that the $\totalassignment$'s assign values to all the Boolean (random) variables in $\phi$ we finally have:
\begin{align}
    \prob(\phi)
    =
    \sum_{\totalassignment} 
    \ive{\phi \models \totalassignment}
    \prod_{w_i \in \totalassignment} \prob( w_i),
    \label{eq:prob_dsic}
\end{align}
where $\ive{\phi \models \totalassignment}$ is an indicator function that evaluates to one if $\phi$ models the value assignment $\totalassignment$ and zero otherwise.
In a neurosymbolic setting \citep{manhaeve2018deepproblog} we additionally have the presence of some subsymbolic data $\xvars$  in the conditioning set:
\begin{align}
    \prob_\params(\phi \mid \xvars)
    =
    \sum_{\totalassignment} 
    \ive{\phi \models \totalassignment}
    \prod_{w_i \in \totalassignment} \prob_{\params}( w_i \mid \xvars).
    \label{eq:nesy_dsic}
\end{align}
Here we use a neural parametrization for $\prob_{\params}( w_i \mid \xvars)$ depending on the parameters~$\params$. Note that by omitting an index on $\params$ we allow for parameters to be shared between the different probability distributions in the product.
Note also how the \dsic principle manifests itself in Equation~\ref{eq:prob_dsic} and Equation~\ref{eq:nesy_dsic} with the product over probabilities encoding the independent choices and the indicator function representing the deterministic system.

\subsection{Neurosymbolic Learning}

Apart from a few exceptions \citep{misino2022vael,de2023neural,di2020efficient,jiang2020generative} (this being a non-exhaustive list), most works in the \nesy literature to date are concerned with supervised classification problems and use the cross-entropy as a loss function. 

For $K$-class multi-class classification in neurosymbolic AI,
we have $K$ logic formulas $(\phi_k, \dots, \phi_K)$ (one for each class)  whose probabilities we would like to know.
This means that we have a one-to-one mapping between the classes we would like to predict and the logic formulas.
Furthermore, these $K$ formulas are mutually exclusive ($\phi_i \land \phi_j = \bot$, with $i{\neq} j$) and exhaustive ($\lor_{i=1}^{K} \phi_k=\top$). These conditions mirror the setting for traditional supervised classification, where one assumes classes to be mutually exclusive and exhaustive, as well.

Given a data point $(\xvars, \yvar)$, where $\xvars$ are the (subsymbolic) features and $\yvar$ is the class label, we write the cross-entropy as
\begin{align}
    \loss(\params, \xvars, \yvar)
    = 
    - \sum_{k=1}^{K} \ive{\yvar=k} \log \prob_\params(\phi_k \mid \xvars).
    \label{eq:ce_nesy}
\end{align}
In the special case of binary classification this reduces to
\begin{align}
    \loss(\params, \xvars, \yvar)
    &
    = 
    - \ive{\yvar=1} \log \prob_\params(\phi_1 \mid \xvars)
    - \ive{\yvar=0} \log \prob_\params(\phi_0 \mid \xvars)
    \nonumber
    \\
    &
    = 
    - \ive{\yvar=1} \log \prob_\params(\phi_1 \mid \xvars)
    - \ive{\yvar=0} \log \bigl(1- \prob_\params(\phi_1 \mid \xvars)\bigr),
    \label{eq:binary_ce_nesy}
\end{align}
where we use $1$ to denote the positive class and $0$ to denote the negative class. Manhaeve et al. \cite{manhaeve2018deepproblog} used this form to enforce logical constraints on the outputs of neural networks and Xu et al. \cite{xu2018semantic} used it to regularize neural networks. The latter work also coined the term \textit{semantic loss}.
Consequently, we will denote the neurosymbolic loss in Equation~\ref{eq:ce_nesy} by $\loss_{SL}(\params, \xvars, \yvar)$
for the remainder of the paper.

The \nesy learning problem (for the general case) now consists of finding the parameters $\params^*$ by performing the following optimization:
\begin{align}
    \params^*
    =
    \argmin_\params 
    \sum_{(\xvars, \yvar) \in \dataset}
    \loss_{SL}(\params, \xvars, \yvar),
\end{align}
where $\dataset$ denotes a dataset of features-label tuples. Once we have found $\params^*$ we can perform neurosymbolic classification with:
\begin{align}
    \hat{\yvar}(\xvars) = \argmax_{k\in \{1,\dots K\}} \prob_{\params^*} (\phi_k \mid \xvars).
\end{align}
Note that in practice, we will not be able to find the globally optimal parameters $\params^*$ as the optimization problem is non-convex.


\section{Neurosymbolic AI and Disjunctive Supervision}
\label{sec:theory}

In the classical supervised machine learning setting \cite{vapnik1995nature} we are given an input-output pair $(\xvars, \yvar)$. The goal is then to learn a model that predicts the output $\yvar$ from the inputs $\xvars$ \cite{vapnik1995nature}.
This classical setting has been generalized in various forms. For instance, to partial label learning (PLL) \cite{grandvalet2004learning,come2009learning,jin2002learning,JMLR:v12:cour11a}. PLL is a type of weakly supervised learning where each training instance is associated with a set of candidate labels, but only one of them is the true label and the specific correct label within the set is unknown.


By further relaxing the assumption in PLL that only a single candidate label is true, we obtain the learning setting of what Zombori et al. \cite{zombori2024towards} call \textit{disjunctive supervision}.

\begin{definition}
\label{def:disjsup}
    [Disjunctive Supervision \cite{zombori2024towards}] Given an input-output pair $(\xvars, \tilde{\yvars})$, with $\xvars$ denoting a real-valued vector ($\xvars \in \mathbb{R}^N$) and $\tilde{\yvars}$ denoting a bit vector ($\tilde{\yvars} = (\tilde{\yvar}_1, \dots, \tilde{\yvar}_M) \in \mathbb{B}^M$), we define the disjunctive supervision loss as 
    \begin{align}
        \loss_{DS}(\params, \xvars, \tilde{\yvars})
        =
        -
        \log \sum_{m=1}^{M} \prob_\params (m \mid  \xvars) \tilde{\yvar}_m,
        \label{eq:disj_super}
    \end{align}
    where $\prob_\params (m \mid  \xvars)$ satisfies $\sum_{m=1}^{M} \prob_\params (m \mid  \xvars)=1$.
\end{definition}

At first glance Zombori et al.'s disjunctive supervision loss looks quite different from the semantic loss used in \nesy AI. However, we show next that the semantic loss follows as a special case.

\begin{theorem}
Neurosymbolic classification (\cf Equation~\ref{eq:ce_nesy}) is a special case of disjunctive supervision (\cf Equation~\ref{eq:disj_super}).
\end{theorem}

\begin{proof}
We start by plugging in Equation~\ref{eq:nesy_dsic} into the semantic loss (Equation~\ref{eq:ce_nesy}):
\begin{align}
\loss_{SL}(\params, \xvars, \yvar)
    &=
    -
    \sum_{k=1}^{K}
    \ive{\yvar=k}
    \log
    \sum_{\totalassignment} 
    \ive{\phi_k \models \totalassignment}
    \prod_{w_i \in \totalassignment} \prob_{\params}( w_i \mid \xvars)
    \nonumber
    \\
    &=
    -
    \sum_{k=1}^{K}
    \ive{\yvar=k}
    \log
    \sum_{\totalassignment} 
    \ive{\phi_k \models \totalassignment}
    \prob_{\params}( \totalassignment \mid \xvars)
    \nonumber
    \\
    &=
    -
    \log
    \sum_{\totalassignment} 
    \ive{\phi_y \models \totalassignment}
    \prob_{\params}( \totalassignment \mid \xvars)
    \label{eq:first_simplification_learning}
\end{align}
Going from the first to the second line the assumption of conditional independence is dropped, then from the second to the third line the indicator is summed out.
Next, let us assume that we have $M$ possible worlds $\totalassignment$ and that we identify each of the worlds using an index $m$.
\begin{align}
    \loss_{SL}(\params, \xvars, \yvar)
    =
    -
    \log
    \sum_{m=1}^{M } 
    \ive{\phi_y \models \totalassignment_m}
    \prob_{\params}( \totalassignment_m \mid \xvars)
\end{align}
We can now think of $\ive{\phi_y \models \totalassignment_m}$ and 
$\prob_{\params}( \totalassignment_m \mid \xvars)$ as vectors being indexed with $m$, and more specifically even we can think of $\ive{\phi_y \models \totalassignment_m}$ as a bit vector that is $1$ for those entries where the condition holds and $0$ otherwise. 
We denote this bit vector by $\boldsymbol{\beta}_{y} =(\beta_{y 0}, \dots, \beta_{y M})$:
\begin{align}
    \loss_{SL}(\params, \xvars, \yvar)
    =
    -
    \log
    \sum_{m=1}^{M } 
    \beta_{y m}
    \prob_{\params}( \totalassignment_m \mid \xvars).
\end{align}
By simply replacing $y$ with $\boldsymbol{\beta}_{y}$ in the input of the loss and by noting that there is a one-to-one correspondence between $\totalassignment$ and $m$ we write:
\begin{align}
    \loss_{SL}(\params, \xvars, \boldsymbol{\beta}_{y})
    =
    -
    \log
    \sum_{m=1}^{M } 
    \beta_{y m}
    \prob_{\params}( m \mid \xvars).
\end{align}
By identifying $ \boldsymbol{\beta}_{y}$ with $\tilde{\yvars}$ we recover the disjunctive supervision loss. Note, however, that we have certain restrictions not present in Definition~\ref{def:disjsup}. For one, we have that $ \prob_{\params}( \totalassignment \mid \xvars)=\prod_{w_i \in \totalassignment} \prob_{\params}( w_i \mid \xvars)$. Secondly, in the case of the neurosymbolic loss we know that the formulas $\phi_k$ ($k \in \{1,\dots, K \}$) are mutually exclusive and exhaustive. This means that the $\boldsymbol{\beta}_{y}$ are orthogonal to each other and complete. In other words, $\sum_{m=1}^M \boldsymbol{\beta}_{y_1 m} \boldsymbol{\beta}_{y_2 m} = 0$ ($y_1\neq y_2$) and $\sum_{y=1}^K \boldsymbol{\beta}_{y} = (1,\dots 1)$. Given that these restrictions are not necessary for disjunctive supervision we conclude that indeed the semantic loss is a special case of the disjunctive supervision loss.
\qed
\end{proof}

\subsection*{The Winner-Take-All Effect}

In order to enforce the constraint that $\sum_{m=1}^M \prob_\params(m \mid \xvars)=1$, a common choice in partial label learning, and also disjunctive supervision learning, is to parametrize $\prob_\params(m \mid \xvars)$ using a neural network with $M$ output units and passing these through a softmax layer. Zombori et al. showed that this is problematic as shown in their \textit{Winner-Take-All} (WTA) theorem.\footnote{We state a simplified version of the WTA theorem that can be found in the Appendix of \cite{zombori2024towards} under Lemma 19. The full theorem is stated in Theorem 4 of \cite{zombori2024towards}.}
\begin{theorem}[Winner-Take-All \cite{zombori2024towards}]
Let $m$ and  $n$ be two acceptable outputs, i.e. $\tilde{y}_m = \tilde{y}_n = 1$, of a   and let $\params_{t}$ and $\params_{t+1}$ denote the parameter before and after a gradient update. 
Then it holds that
\begin{align}
\frac{\prob_{\params_{t+1}}(m\mid \xvars)}{\prob_{\params_{t+1}}(n \mid \xvars)}
>
\frac{\prob_{\params_t}(m\mid \xvars)}{\prob_{\params_t}(n \mid \xvars)}
\end{align}
exactly when $\prob_{\params_t}(m\mid \xvars)>\prob_{\params_t}(n \mid \xvars)$.
\end{theorem}
The WTA theorem states that the output of the softmax with the maximal initial probability will eventually capture all the probability mass. This is problematic in the  sense that the intended semantics of disjunctive supervision states that the outputs $m$ and $n$ are equally valid and that none should be preferred over the other. However, the output that initially receives the higher probability will typically end up capturing the entire probability mass after optimization -- for no other reason than random initialization.
Note that Zombori et al. have proven this only for a neural network consisting of a single linear layer followed by a softmax. However, they have experimentally shown that the effect is also noticeable in deep networks. We reproduce their results in the next section. 

Although using a softmax as the last layer in disjunctive supervision is not necessary, it seems nevertheless to constitute the de facto standard approach. This stands in contrast to neurosymbolic AI where the probability $\prob_\theta(\phi \mid \xvars)$ is parametrized using a sum over products (\cf Equation~\ref{eq:nesy_dsic}). In the next section, we experimentally show that by using the \dsic parameterization, instead of standard disjunctive supervision parametrization,  we prevent the problematic Winner-Take-All effect from happening.


    \section{Experimental Evaluation}
\label{sec:experiments}

We perform our experimental comparison using the traffic light example introduced by van Krieken et al. \cite{van2024independence}.
Suppose therefore that we are given a traffic light consisting of a red light and a green light. Additionally, we have the constraint that at most one of the lights is switched on.
Given an observation of the traffic light, we would now like to predict the probability that the input indeed satisfies the constraint.

In our experiments, we represent traffic lights using MNIST digits \cite{deng2012mnist}. Specifically, each of the two lights is represented by an MNIST image. If the specific light is on we use an MNIST digit depicting a one. Otherwise we use a zero.
This means we have four possible configurations $\{ (0,0), (1,0),(0,1),(1,1) \}$, which we represent using MNIST digits
$\{(\includegraphics[scale=0.035]{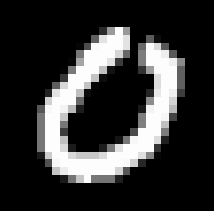}, \includegraphics[scale = 0.035]{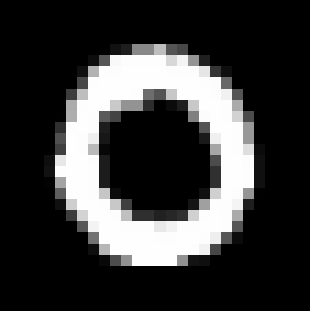}), (\includegraphics[scale=0.016]{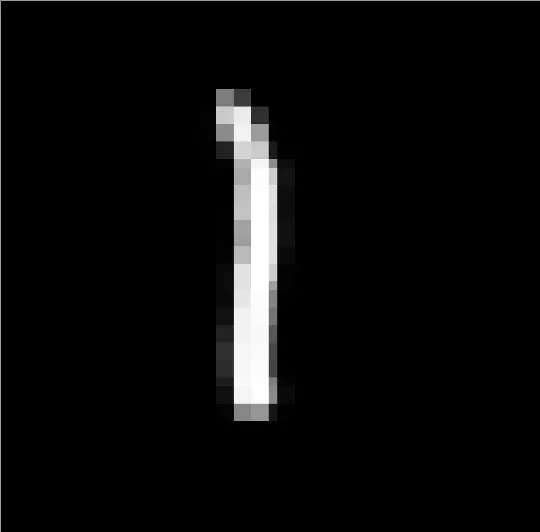}, 
\includegraphics[scale = 0.035]{arxiv/figures/mnist_0.png}), (\includegraphics[scale=0.035]{arxiv/figures/mnist_0.png}, \includegraphics[scale = 0.016]{arxiv/figures/1_PF22RoBwKIRioZjYkVjJ6w.png}), (\includegraphics[scale=0.016]{arxiv/figures/1_PF22RoBwKIRioZjYkVjJ6w.png}, \includegraphics[scale = 0.016]{arxiv/figures/1_PF22RoBwKIRioZjYkVjJ6w.png})\}$. In the \nesy setting these four configurations correspond to four possible worlds. For disjunctive supervision this means that the softmax at the very last layer has four outputs. 
We now describe the two settings in more detail.

\subsection*{The Traffic Light Example Using the Semantic Loss}

We would now like to learn a model that infers whether the constraint on the images is satisfied. Formally, the constraint \say{at most one of the lights are on} can be expressed as 
\begin{equation}
    \phi_1
    \leftrightarrow 
    (\neg\text{red} \land \text{green})
    \lor
    (\text{red} \land  \neg\text{green})
    \lor
    (\neg \text{red} \land \neg \text{green}).
\end{equation}
This means that the semantic loss takes the following form:
\begin{align}
    \loss_{SL}(\params, \xvars, \yvar)
    = 
    - \ive{\yvar=1} \log \prob_\params(\phi_1 \mid \xvars)
    - \ive{\yvar=0} \log \bigl(1- \prob_\params(\phi_1 \mid \xvars)\bigr),
\end{align}
where we use the semantic loss for binary classification, \cf Equation~\ref{eq:binary_ce_nesy}.
Furthermore, we compute the probability $\prob_\params(\phi_1 \mid \xvars)$ as:
\begin{align}
    \prob_\params(\phi_1 \mid \xvars)
    &=
    \prob_{\params_r}(\text{red} \mid \xvars_r) (1 {-}\prob_{\params_g}(\text{green} \mid \xvars_r))
    +
    \nonumber
    \\
    &
    \quad
    (1- \prob_{\params_r}(\text{red}\mid \xvars_r))  \prob_{\params_g}(\text{green} \mid \xvars_g)
    +
    \nonumber
    \\
    &
    \quad
    (1 {-} \prob_{\params_r}(\text{red} \mid \xvars_r))  (1 {-} \prob_{\params_g}(\text{green} \mid \xvars_g)),
\end{align}
with $\xvars = \xvars_r \cup \xvars_g$ denoting the subsymbolic data (MNIST image) for the red and green lights, and $\params = \params_r \cup \params_g$ denoting the parameters for the two neural networks that we use to predict the probability of the red and green light being switched on.  

This means that we have two neural networks parametrized by $\params_r$ and $\params_g$, respectively, and that take either $\xvars_r$ or $\xvars_g$ as input. Using a sigmoid in the last layer of both networks we ensure that we encode proper probability distributions $\prob_{\params_r}(\text{red}\mid \xvars_r)$ and $\prob_{\params_g}(\text{green} \mid \xvars_g)$.

\subsection*{The Traffic Light Example Using Disjunctive Supervision}

As the traffic light example has four possible configurations we parametrize a neural encoding for a probability distributions with four possible outcomes:
\begin{align}
    p_\params ( m \mid \xvars_r, \xvars_g ), \quad m \in \{1,2,3,4\}.
\end{align}
We ensure that we have a proper probability distribution using a softmax in the final layer.
The bit vectors $\tilde{\yvars}$ that we use for disjunctive supervision are:
\begin{align}
    \tilde{\yvars}_1 &= (1,1,1,0)
    \label{eq:bitvec_pos_traffic}
    \\
    \tilde{\yvars}_0 &= (0,0,0,1).
    \label{eq:bitvec_neg_traffic}
\end{align}
We use $\tilde{\yvars}_1$ for positive training examples, \ie when the images $\xvars_r$ and $\xvars_g$ indeed obey the constraints and $\tilde{\yvars}_0$ when the constraints are violated, \ie both images $\xvars_r$ and $\xvars_g$ show an MNIST one. The choice made of encoding the labels in Equation~\ref{eq:bitvec_pos_traffic} (positive) and Equation~\ref{eq:bitvec_neg_traffic} (negative) means that we associate the first three outputs of the neural network with the configurations that satisfy the traffic lights constraint (referred to as Possible World 1, 2 and 3 in Figure \ref{fig:plots_ds}), while the last one corresponds to the non-satisfying case (referred to as the Impossible world in Figure \ref{fig:plots_ds}).

\subsection*{Experimental Questions}

\begin{description}
    \item[Q1] Does the traffic light problem exhibit the Winner-Take-All effect?
    \item[Q2] Is there a qualitative difference in training behavior between the semantic loss and disjunctive supervision?
\end{description}


\subsection{Experimental Setup}

\paragraph{Dataset.}

As already mentioned we represent the traffic lights example using MNIST digits depicting ones and zeros. Given the label of the individual digits we check whether the constraint that at most one light is on is satisfied. The two images with MNIST digits and the label of the constraint being satisfied or not then constitute a data point.


Since three of the four configurations result in the constraint being satisfied, there is a data imbalance between positive and negative examples if pairs of images are sampled uniformly.
We, therefore, oversampled the configuration where both images show a one (negative case) to compensate for this data imbalance.

\paragraph{Training parameters.}

We implemented all our experiments in DeepProbLog \cite{manhaeve2018deepproblog}. In order to train the neural networks we used the Adam optimizer with a learning rate of $0.001$ and batch size of $32$. The number of training examples was $3200$ and the test size was $200$, \ie $50$ per possible configuration.


\subsection{Does the Traffic Light Problem Exhibit the Winner-Take-All Effect?}

The purpose of this experiment is to determine whether a single neural network, which receives as input two MNIST images and outputs a single distribution over four possible configurations of the traffic light example exhibits the WTA effect, \ie the probability will be concentrated in only one of the possible configurations.

\paragraph{Neural network architecture.} The neural network structure we used has the following form; each image is processed independently through a CNN with two convolutional layers (5×5 kernels, 6 and 16 channels), followed by ReLU activation and 2×2 max pooling. The extracted features were then flattened, concatenated, and passed through a fully connected classifier with layers of 120, 84, and 4 fully connected layers, using ReLU activations and a final softmax for classification.


\begin{figure*}[t!]
        \subfloat[$\neg\text{red}\land\neg\text{green}$]{%
            \includegraphics[width=.48\linewidth]{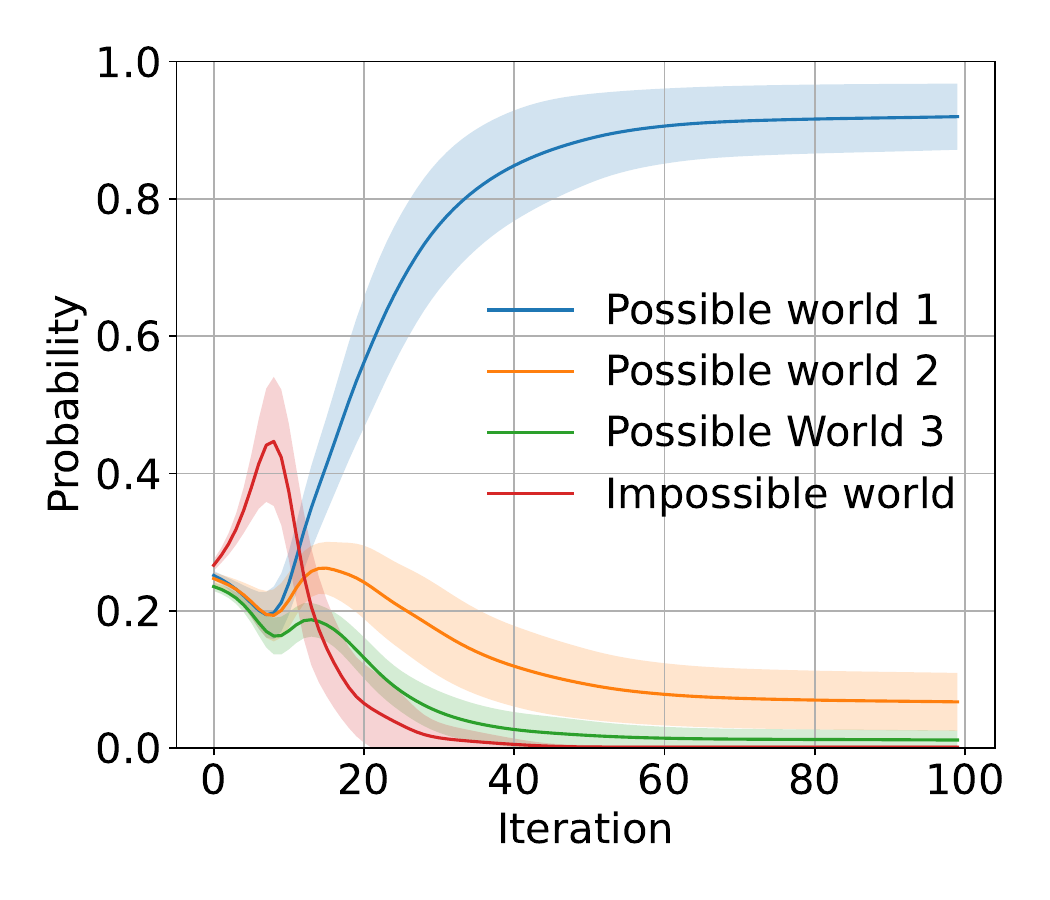}%
            \label{subfig:zsolt_a}%
        }\hfill
        \subfloat[$\neg\text{red}\land\text{green}$]{%
            \includegraphics[width=.48\linewidth]{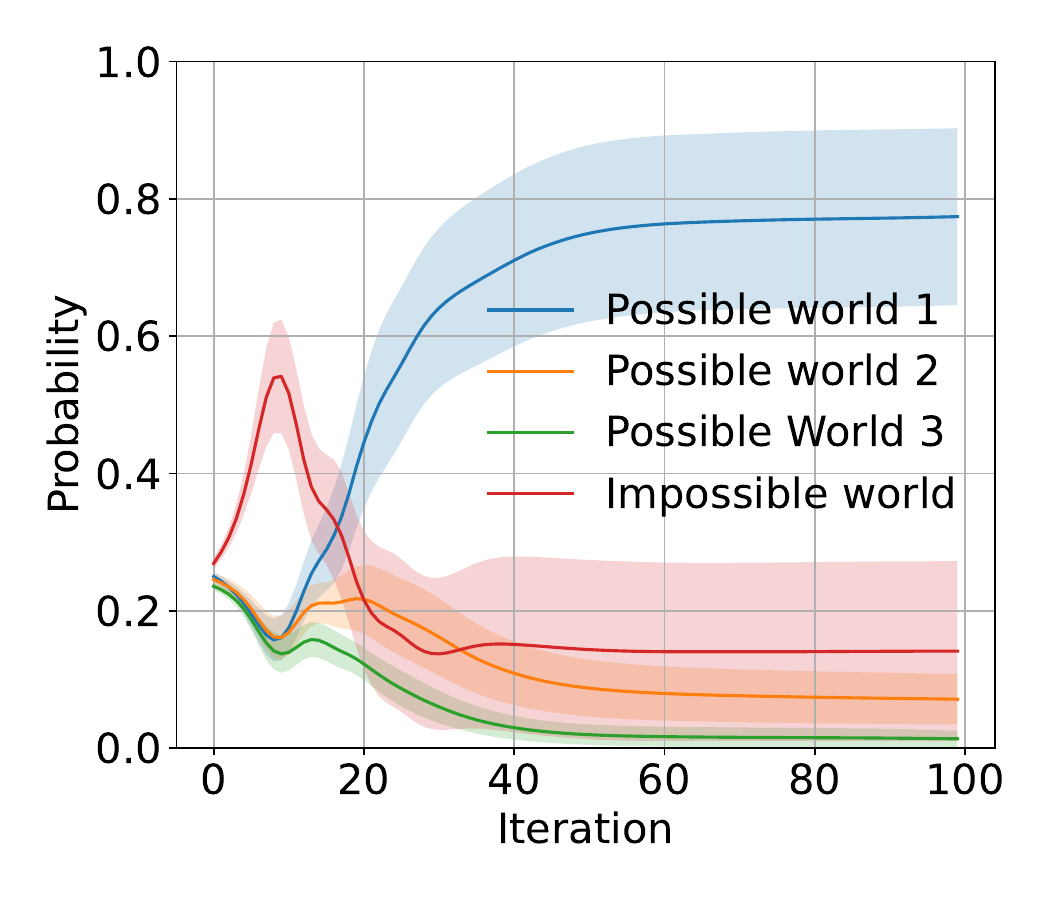}%
            \label{subfig:zsolt_b}%
        }\\
        \subfloat[$\text{red}\land\neg\text{green}$]{%
            \includegraphics[width=.48\linewidth]{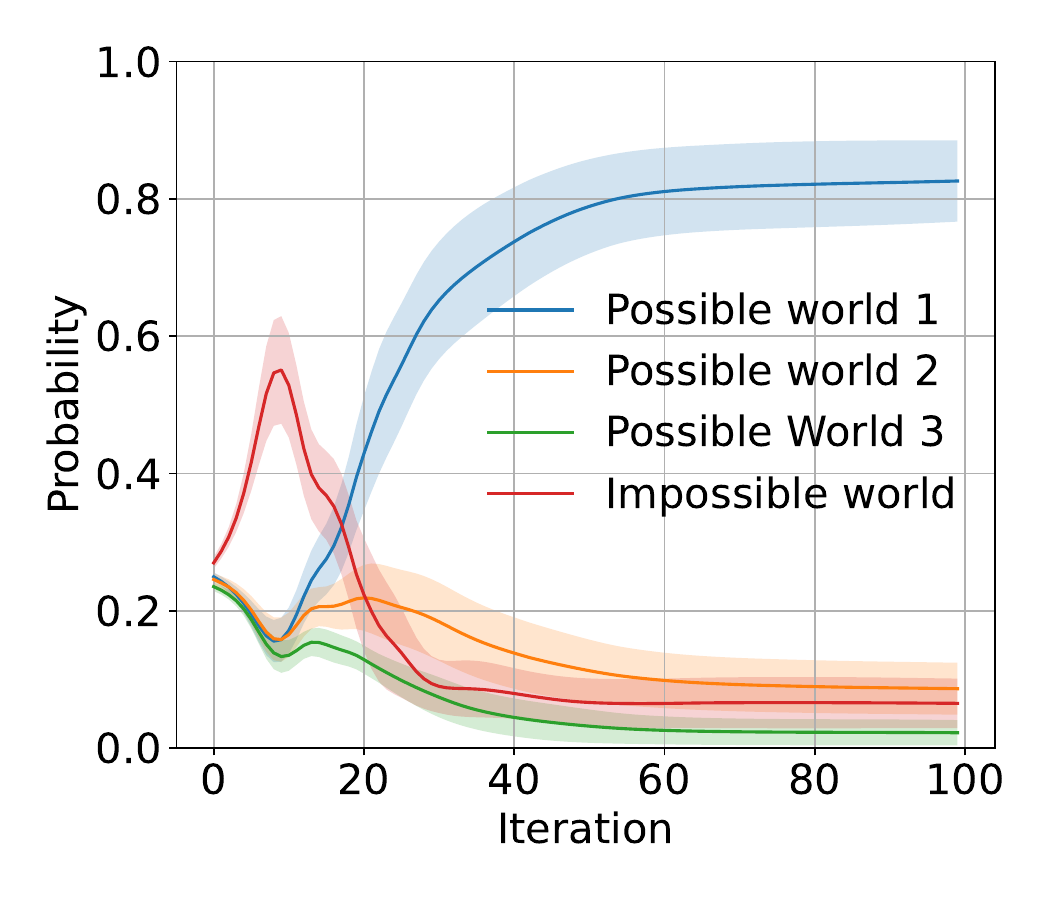}%
            \label{subfig:zsolt_c}%
        }\hfill
        \subfloat[$\text{red}\land\text{green}$]{%
            \includegraphics[width=.48\linewidth]{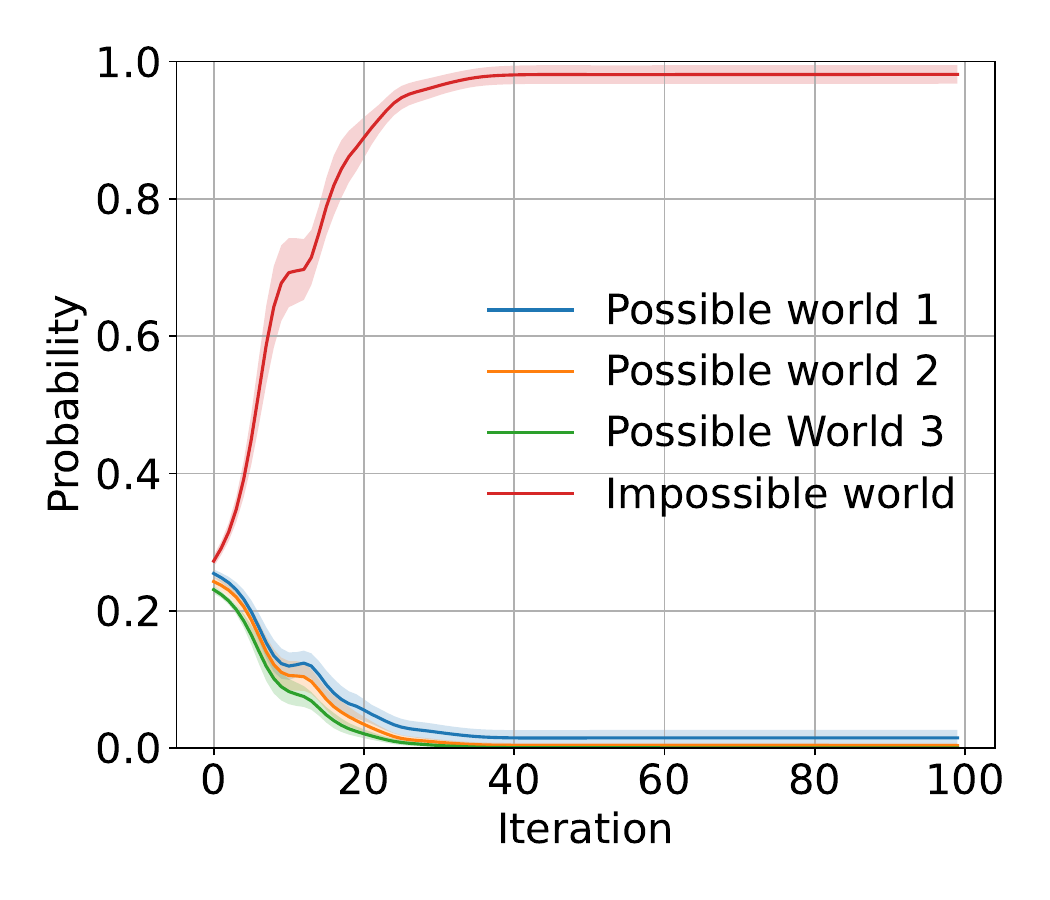}%
            \label{subfig:zsolt_d}%
        }
        \caption{Experimental evaluation with disjunctive supervision. The plots show the mean value of the probability at each iteration over 20 runs. 
        The shaded areas indicate the $95\%$ confidence intervals.
        We report the mean values separately for the four different parts of the test set. The labels under each graph indicate which part is being evaluated, \eg $\neg\text{red}\land\text{green}$ corresponds to having as input an MNIST digit depicting a zero and an MNIST image depicting a one. It can be seen that for the possible cases Figure \ref{subfig:zsolt_a} to Figure \ref{subfig:zsolt_c} one of the worlds is dominating, whilst the other ones go towards $0$ as the training iterations increase. For the impossible case \ref{subfig:zsolt_d}, the impossible world is dominant. Note that we rank the first three outputs of the softmax according to the sum of their probability over the entire run, this allows us to identify the individual outputs.}
        \label{fig:plots_ds}
    \end{figure*}

\paragraph{Results.} 
We pass,
after each update of the parameters using the disjunctive supervision loss,
the elements of the test set through the updated model. We then record the probability of the four outputs, \ie worlds.
We split this analysis into four parts, one for each of the possible cases in the test set ($\neg\text{red}\land\text{green}$, $\text{red}\land\neg\text{green}$, $\neg\text{red}\land\neg\text{green}$, and $\text{red}\land\text{green}$), where the first three adhere to the constraint of the traffic light example and the last one does not.

In Figure~\ref{fig:plots_ds} we observe the confirmation of Zombori et al.'s WTA theorem. That is, one of the outputs of the softmax captures a considerable share of the probability mass regardless of what the input is. We can observe this in Figure~\ref{subfig:zsolt_a}, Figure \ref{subfig:zsolt_b}, and  Figure \ref{subfig:zsolt_c}. Specifically, for those inputs that adhere to the constraint, the model favors one of the three possible outputs. \ie all inputs map to the same output of the softmax. Only for the case that the test example does not adhere to the constraint, \ie $\text{red}\land\text{green}$, we have a different behavior, \cf Figure~\ref{subfig:zsolt_d}.





\subsection{Is there a qualitative difference in training behavior between the semantic loss and disjunctive supervision?}
\label{sec:experiment_semanticloss}

In this experiment we use the semantic loss with two independent distributions $\prob_{\params_r}(red \mid \xvars_r)$ and $\prob_{\params_g}(green \mid \xvars_g)$.
The purpose of this experiment is to examine whether the classical neurosymbolic approach also exhibits a WTA effect.

\paragraph{Neural network architecture.}  The neural network used to parametrize $\prob_{\params_r}(red \mid \xvars_r)$ and $\prob_{\params_g}(green \mid \xvars_g)$ consisted of a convolutional encoder and a fully connected classifier. The encoder applies two convolutional layers (5×5 kernels, 6 and 16 channels) with ReLU activations, each followed by 2×2 max pooling. After passing through the encoder, the feature representation is flattened and passed through a classifier with fully connected layers of 120 and 84 units, both using ReLU activations, followed by a final layer with two output units and a sigmoid activation for binary classification.

\paragraph{Results}
We report the results in Figure~\ref{fig:nesy_cases}, where we break up the analysis again into four parts. Note that because we have explicit predictions for the probabilities $\prob_{\params_r}(red \mid \xvars_r)$ and $\prob_{\params_g}(green \mid \xvars_g)$ we are now able to identify exactly which world we are in. This is reflected in the legend of the plots. We see that for all four cases the neurosymbolic model predicts with high probability the correct world. For instance, in Figure~\ref{subfig:nesy_a} the model receives at each iteration those elements of the test set for which the MNIST digits correspond to the red and green light being zeros. As can be seen, the trained model correctly predicts the world, which means that the neural networks classifying the individual digits learn to classify the MNIST digits without receiving direct supervision but only receiving supervision on whether the constraint of the traffic light example is satisfied or not.

\begin{figure*}[t!]
        \subfloat[$\neg\text{red}\land\neg\text{green}$]{%
            \includegraphics[width=.48\linewidth]{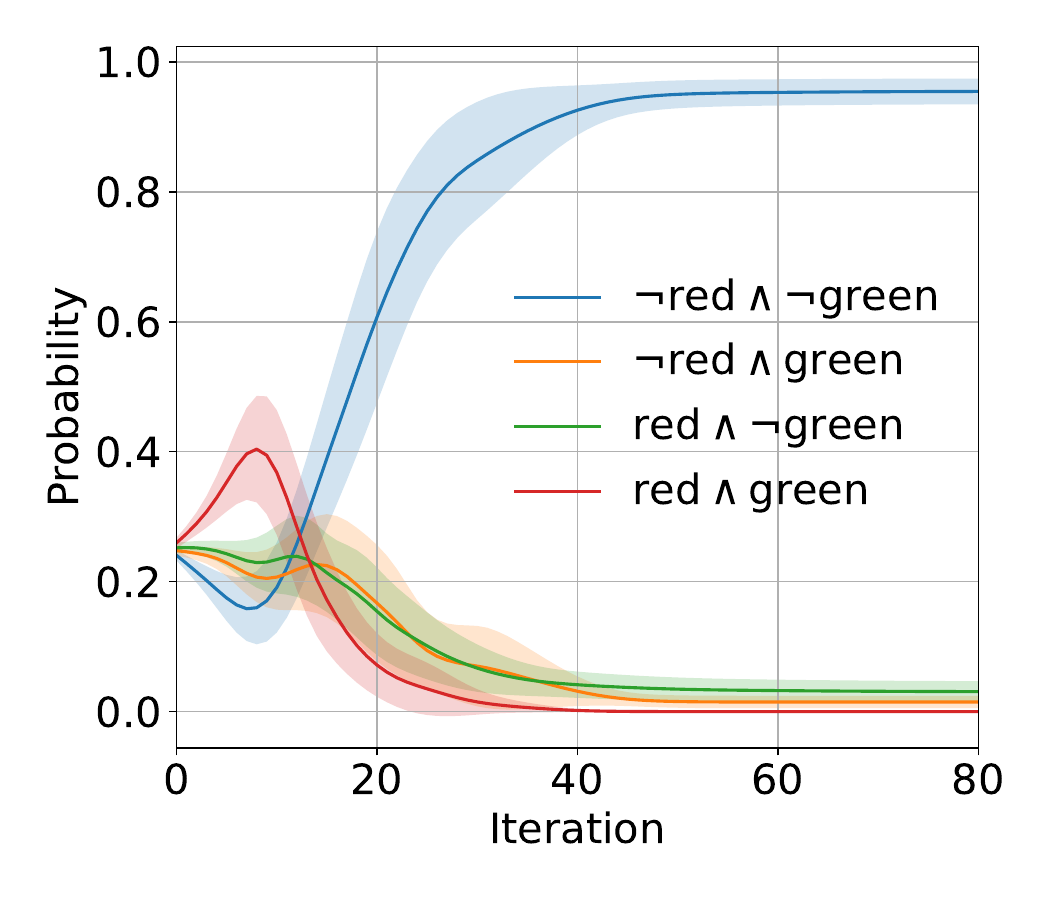}%
            \label{subfig:nesy_a}%
        }\hfill
        \subfloat[$\neg\text{red}\land\text{green}$]{%
            \includegraphics[width=.48\linewidth]{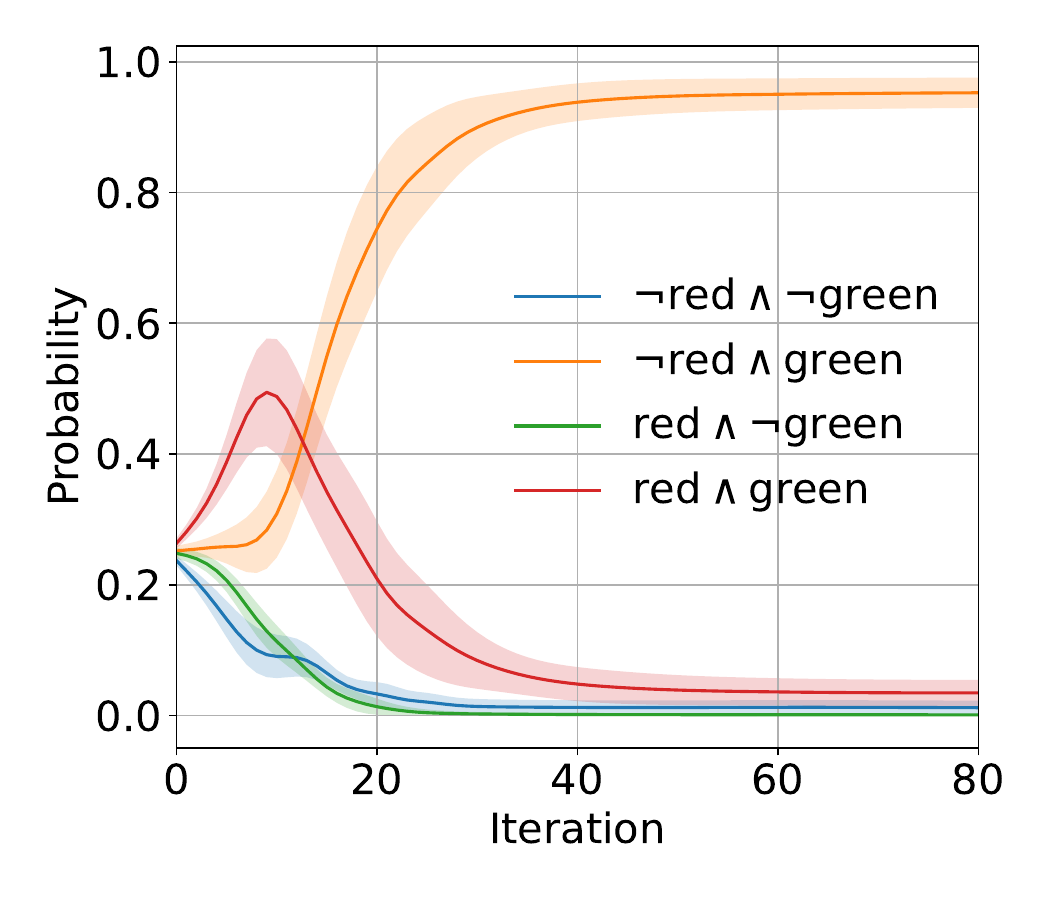}%
            \label{subfig:nesy_b}%
        }\\
        \subfloat[$\text{red}\land\neg\text{green}$]{%
            \includegraphics[width=.48\linewidth]{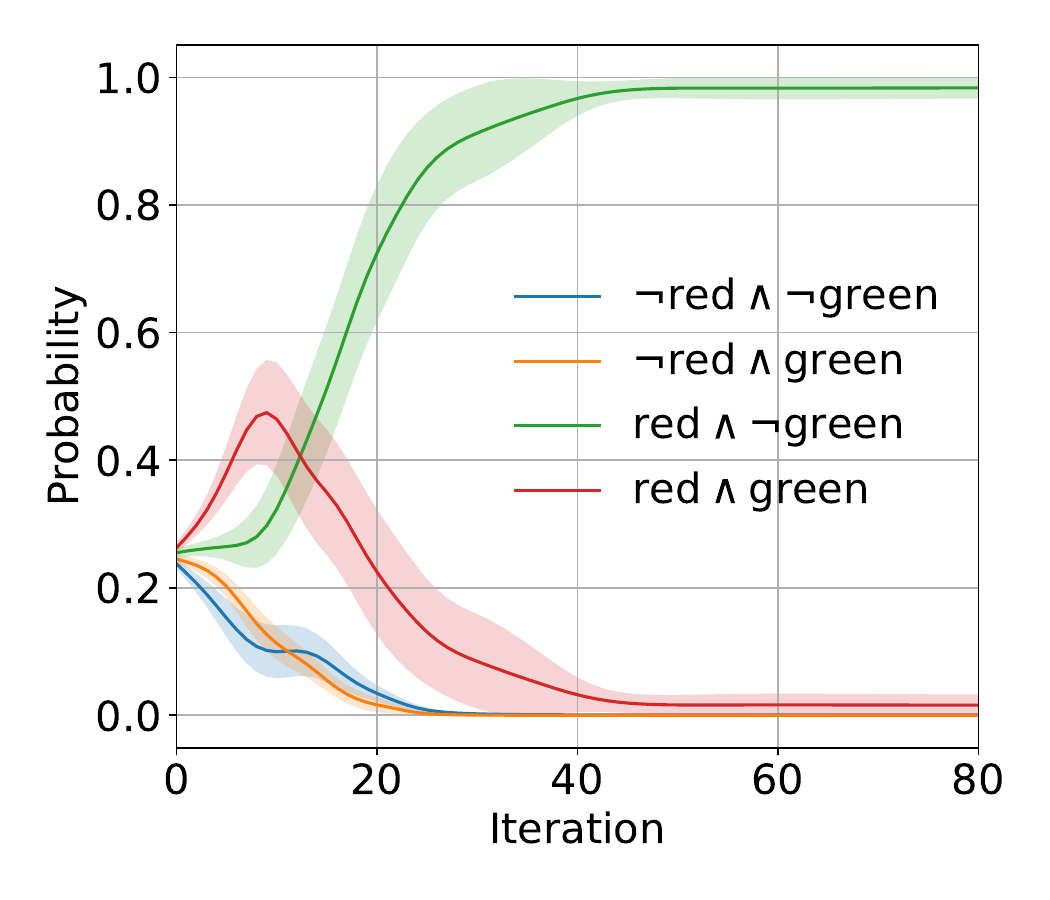}%
            \label{subfig:nesy_c}%
        }\hfill
        \subfloat[$\text{red}\land\text{green}$]{%
            \includegraphics[width=.48\linewidth]{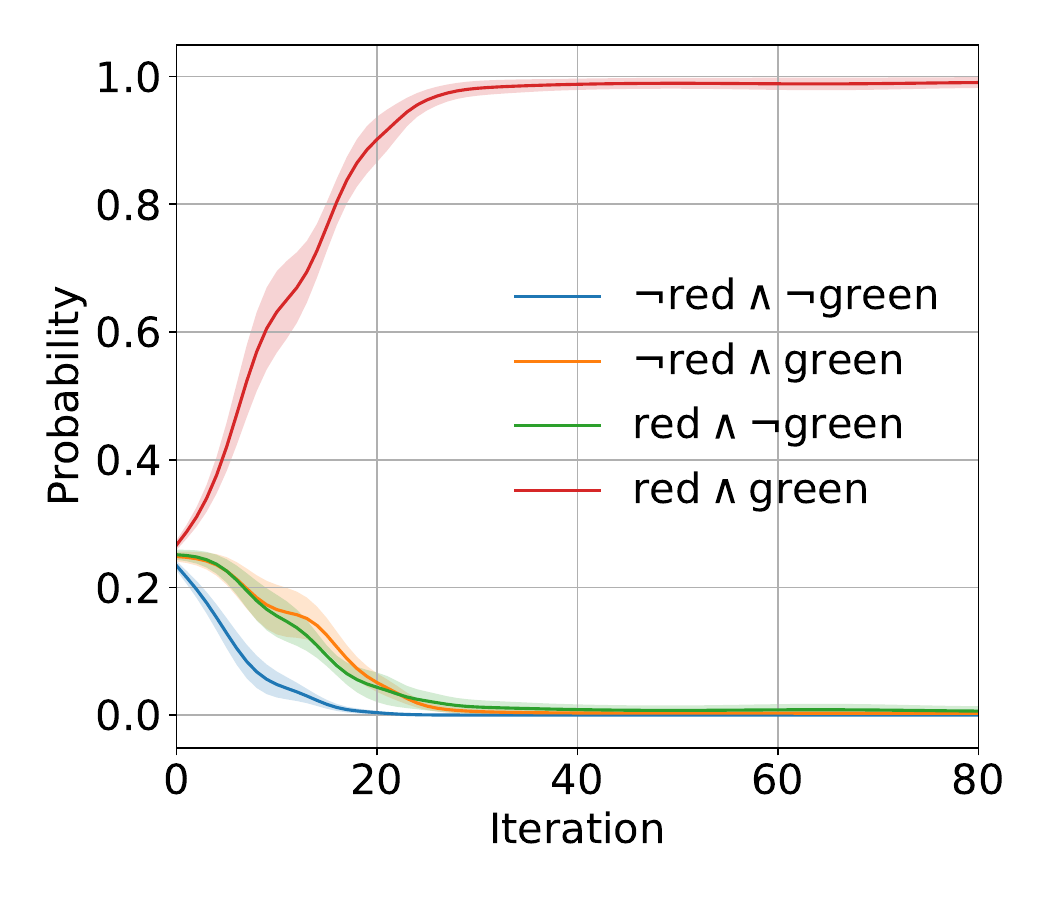}%
            \label{subfig:nesy_d}%
        }
        \caption{Empirical evaluation for the traffic lights example using the semantic loss.
        We split the evaluation again into four parts, one for each of the possible configurations and report again the mean probability over runs on the respective part of the the test set during training. 
        The plots clearly show that as training proceeds only the world that corresponds to the specific part of the test set is expressed. This also applies to the impossible case (Figure~\ref{subfig:nesy_d}.
        }
        \label{fig:nesy_cases}
    \end{figure*}

\subsection{Discussion}
Intuitively one would expect that the more general model (disjunctive supervision) would yield better results than the more restricted model (semantic loss over a conditionally factorized distribution). However, our experimental evaluation suggests the contrary. Specifically, by restricting the model class to the neurosymbolic setting, we are able to avoid the WTA effect that is inevitably present in disjunctive supervision scenarios with softmax functions in the final layer of the neural network.

We attribute the vastly different learning dynamics under neurosymbolic supervision and disjunctive supervision to the fact that constraints in the disjunctive supervision case are not specific enough to allow the model to distinguish between different worlds, \eg distinguish between $\text{red}\land \neg \text{green}$ and $\neg \text{red}\land \text{green}$.

This is akin to a phenomenon observed in neurosymbolic AI dubbed \textit{reasoning shortcuts} \cite{MANHAEVE2021103504,marconato2023neuro}. Reasoning shortcuts occur when the constraints used to supervise the outputs of a neural network do not provide enough information. The neural network will then exploit this underspecification and find a solution that formally satisfies the constraints but does not adhere to the intended meaning. We refer the reader to \cite{marconato2023neuro} for a more detailed discussion on reasoning shortcuts.







\section{A Post-Mortem on the Deterministic Bias}
\label{sec:mortem}

In the previous section we have shown that the neurosymbolic model using conditionally independent distributions $\prob_{\params_r}(red \mid \xvars_r)$ and $\prob_{\params_g}(green \mid \xvars_g)$ does not exhibit the pathologic WTA effect and learns to solve almost perfectly the traffic light example. At first glance, this seems to be somewhat odd as van Krieken et al. introduced the traffic light example to demonstrate the shortcomings of the conditional independence assumption in neurosymbolic AI.

The resolution to this apparent conundrum is rather trivial: instead of studying the semantic loss, van Krieken et al. investigated a \say{truncated} semantic loss (see Section 2 of \cite{van2024independence}). Specifically (and restricting ourselves to the binary classification case) van Krieken et al. chose to optimize the following objective:
\begin{align}
    \loss_{MSL}(\params, \xvars, \yvar)
    = 
    - \ive{\yvar=1} \log \prob_\params(\phi_1 \mid \xvars),
    \label{eq:mutliated_semanticloss}
\end{align}
to which they referred as semantic loss. Comparing this to the actual semantic loss
\begin{align}
    \loss_{SL}(\params, \xvars, \yvar)
    = 
    - \ive{\yvar=1} \log \prob_\params(\phi_1 \mid \xvars)
    - \ive{\yvar=0} \log \bigl(1- \prob_\params(\phi_1 \mid \xvars)\bigr),
\end{align}
the differences are obvious: van Krieken et al. did not include any negative training examples.

We repeated our experiment from Section~\ref{sec:experiment_semanticloss} using van Krieken et al.'s truncated semantic loss, and indeed in this case we observe the deterministic bias effect. We report our results in Figure~\ref{fig:emile_cases}. For the sake of completeness we also give a formal definition of deterministic bias (which was not formally defined in van Krieken et al.'s work). 

\begin{definition}[Deterministic Bias]
\label{def:deterministicbias}
    Let $\loss$ be a loss function and $\prob_{\params_t}$ a para-metrized probability distribution.
    We call the pair $(\loss, \prob_{\params_t})$ deterministically biased if $\prob_{\params_t}(x)\rightarrow C$, with $C \in \{0,1\}$, as $t\rightarrow \infty$, with $t$ denoting the gradient updates.
\end{definition}

We note the striking resemblance (in spirit) to Zombori et al.'s WTA theorem, as both suggests that the models will unjustly favor one of the solutions over the others. Zombori et al.'s WTA theorem, however, applies to the setting when there is a single softmax distribution over worlds, while van Krieken et al.'s bias only applies to cases when the problem is specified by a truncated semantic loss.

While the experiment reported in Figure~\ref{fig:emile_cases} corroborates the presence of a deterministic bias when using the loss in Equation~\ref{eq:mutliated_semanticloss}, the problem is that this loss does not correspond to the original semantic loss \cite{manhaeve2018deepproblog,xu2018semantic}. 
Contrary, to what van Krieken et al. claim, the presence of deterministic bias is not due to a factorized probability distribution but to not properly applying the semantic loss, \ie including negative training examples.

Furthermore, the analysis of van Krieken et al. does also not apply to other neurosymbolic learning settings, such as semi-supervised learning \cite{xu2018semantic}, as here van Krieken et al.'s loss (Equation~\ref{eq:mutliated_semanticloss}) is not used in isolation but as a regularizing term to the actual loss function.

\begin{figure*}[t!]
        \subfloat[$\neg\text{red}\land\neg\text{green}$]{%
            \includegraphics[width=.48\linewidth]{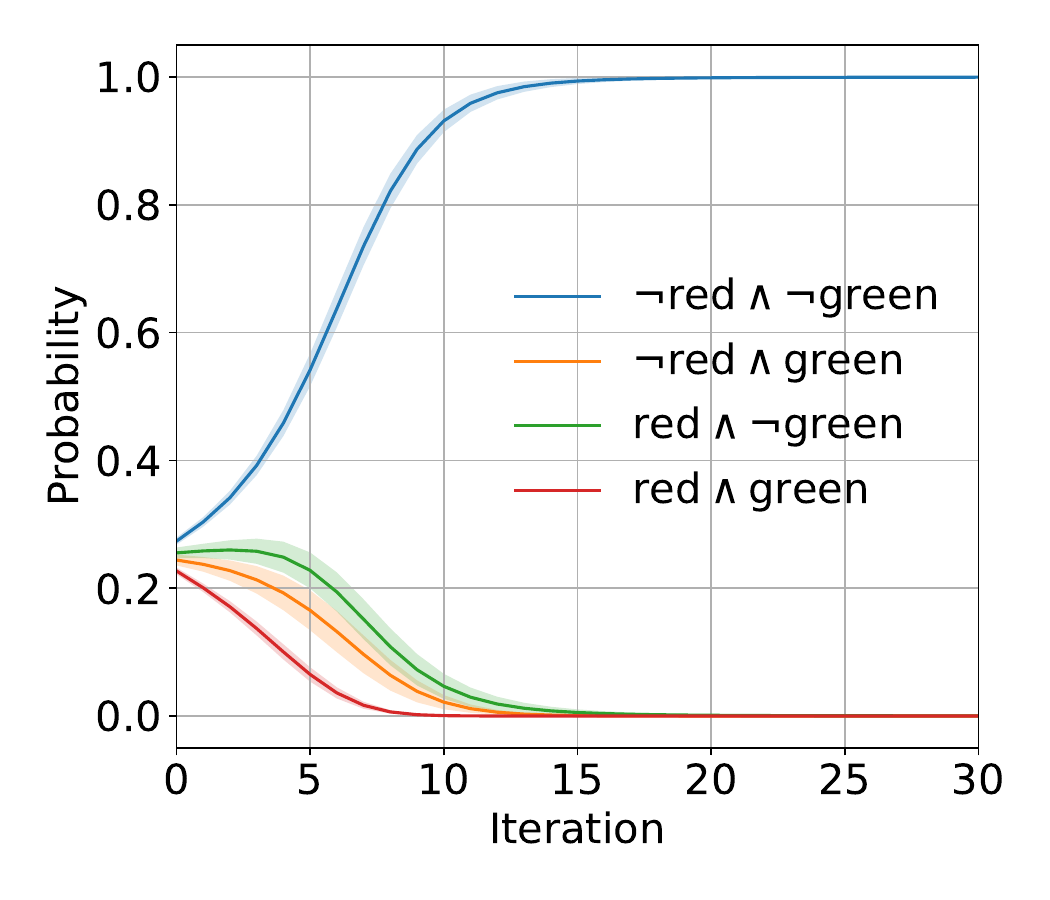}%
            \label{subfig:emile_nesy_a}%
        }\hfill
        \subfloat[$\neg\text{red}\land\text{green}$]{%
            \includegraphics[width=.48\linewidth]{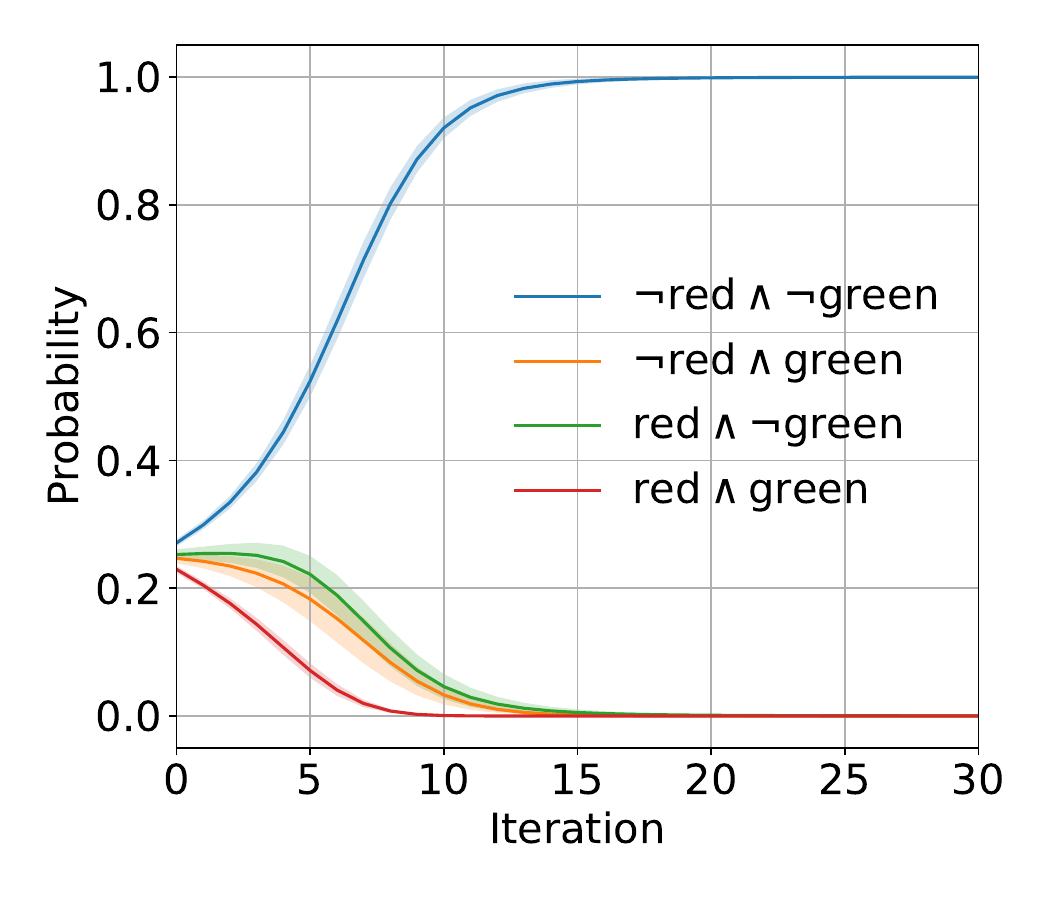}%
            \label{subfig:emile_nesy_b}%
        }\\
        \subfloat[$\text{red}\land\neg\text{green}$]{%
            \includegraphics[width=.48\linewidth]{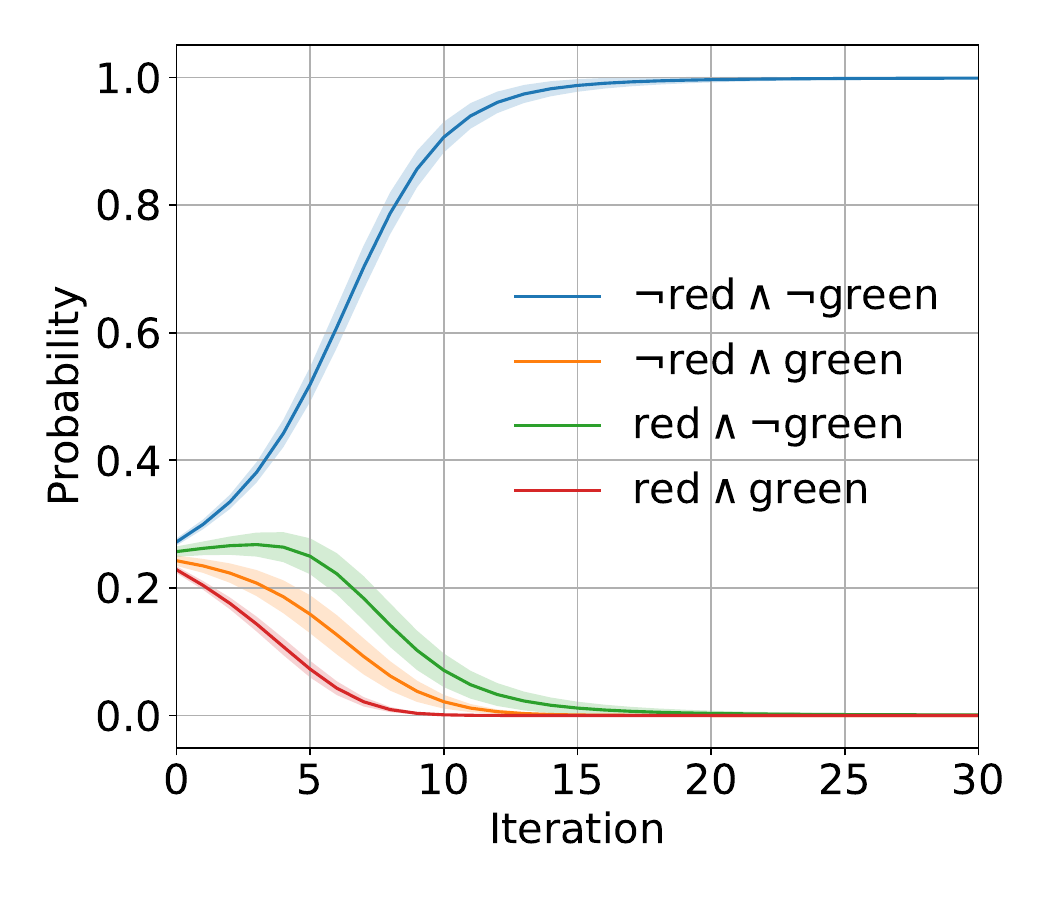}%
            \label{subfig:emile_nesy_c}%
        }\hfill
        \subfloat[$\text{red}\land\text{green}$]{%
            \includegraphics[width=.48\linewidth]{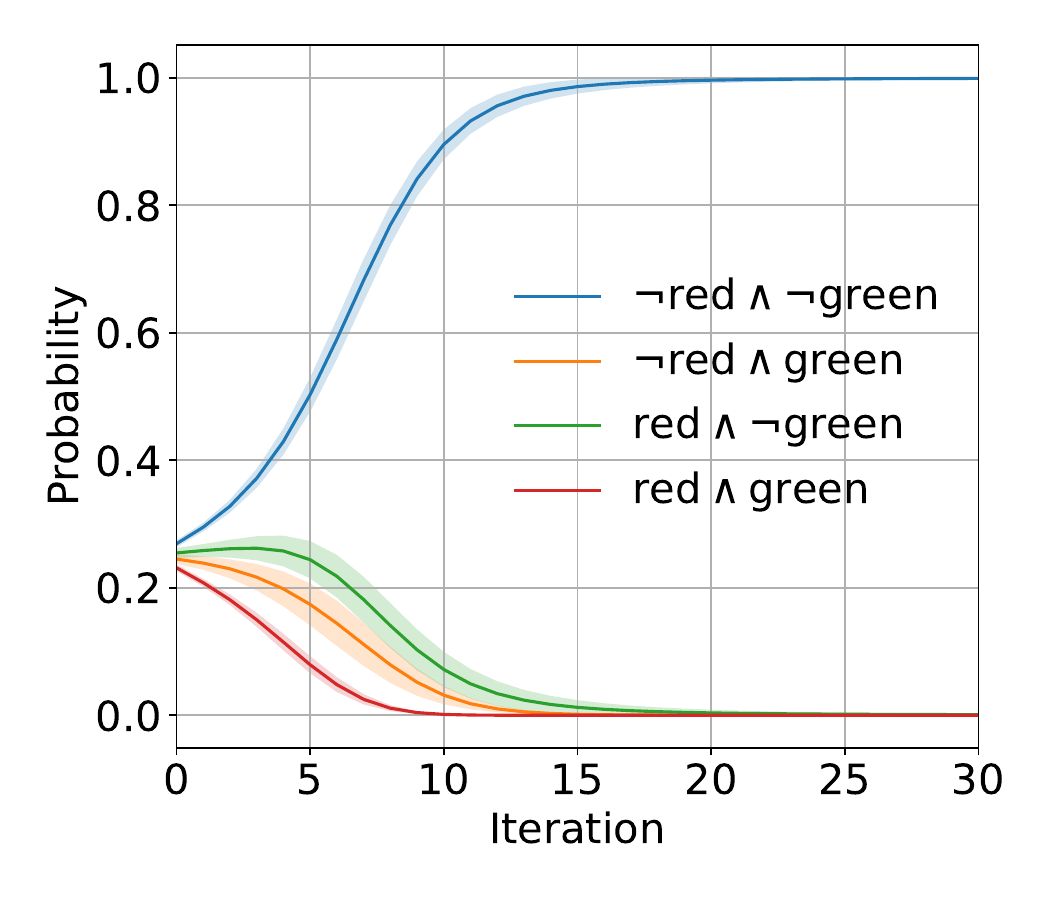}%
            \label{subfig:emile_nesy_d}%
        }
        \caption{This is the empirical evaluation for the traffic lights example using the semantic loss without negative training examples \cite{van2024independence}. For each case only the world when both lights are off is activated, even for the impossible case (Figure~\ref{subfig:emile_nesy_d})
        This means that the model does not learn.}
        \label{fig:emile_cases}
    \end{figure*}



\subsection*{A Note on Conditional Independence}

While we have already mentioned the \dsic principle (Section ~\ref{sec:prelim}) and that we can construct Turing complete probabilistic programming languages with it, we would like to revisit some statements made by van Krieken et al. on the topic of (conditionally) factorized distributions.
Van Krieken et al. argued that restricting probability distributions to distributions exhibiting conditional independence, \ie using
$
\prob_\params(\totalassignment \mid \xvars)
=
\prod_{w_i \in \totalassignment} \prob_{\params}( w_i \mid \xvars)
$
would not allow for representing all possible probability distributions in \nesy systems. However, given that systems such as DeepProbLog \cite{manhaeve2018deepproblog} and DeepSeaProbLog \cite{de2023neural} are strict neural extensions of Turing-complete probabilistic programming languages \cite{de2007problog,dos2024declarative} the factorization of $\prob_\params(\totalassignment \mid \xvars)$ into conditionally independent distributions does not hinder expressive power.

For languages with a finite vocabulary (and therefore not Turing-complete) matters are more nuanced. It was shown that indeed not all distributions can be represented using the \dsic principle \cite{buchman2017negative,buchman2017rules}. Given that conditional probability distributions are nothing but probability distributions with an explicit conditioning set this is also the case for \nesy systems with a finite vocabulary that follow the deterministic systems and \textbf{conditionally} independent choices principle (\dscic). In order to have fully expressive systems, one would need to allow for complex-valued \textit{probability strengths} \cite{buchman2017rules,kuzelka2020complex}.
Alternatively, one can also use non-factorized distributions where one assumes conditional dependencies \cite{dilkas2021weighted,zuidberg2024probabilistic,choi2018relative,shen2019conditional}. Note, these alternatives are not necessary for Turing complete languages, such as DeepProbLog, to express all probability distributions \cite{poole1993probabilistic}.







\section{Conclusions}

We have shown that the semantic loss, which is ubiquitous in neurosymbolic AI, can be viewed as a special case of the disjunctive supervision loss. However, in our experimental evaluation we have also provided evidence that the assumptions made in neurosymbolic AI are beneficial towards learning and avoid the Winner-Take-All effect present when learning with disjunctive supervision.
Curiously, this (conditional) independence assumption, which avoids the WTA effect,
was deemed problematic by van Krieken et al..

We have shown that their conclusions can be traced back to the use of a non-standard definition of semantic loss, in which negative examples are omitted. As a consequence the conclusions  of  van Krieken et al.. do not directly apply to standard neurosymbolic AI.

While learning in the neurosymbolic setting exhibits certain complications compared to standard supervised learning, \eg reasoning shortcuts, we conclude that assuming (conditional) independence between the neurally parametrized probability distributions is not one of them.




\section*{Acknowledgments} The work was supported by the Wallenberg AI, Autonomous Systems and Software Program (WASP) funded by the Knut and Alice Wallenberg Foundation. 
We also thank Luc De Raedt for valuable feedback and insightful discussions,
and Lennert De Smet, Jaron Maene, Giuseppe Marra and Vincent Derkinderen
for helpful feedback and comments on the draft of the paper.

\bibliography{references}


\end{document}